\documentclass[11pt]{article}
\usepackage{acl} 
\usepackage{lmodern}

\usepackage{times}
\usepackage{latexsym}
\usepackage[T1]{fontenc}
\usepackage{float}
\usepackage[utf8]{inputenc}
\usepackage[textsize=tiny]{todonotes}
\usepackage{microtype}

\usepackage{amsfonts}
\usepackage{inconsolata}
\usepackage{subcaption}
\usepackage{booktabs}
\usepackage{amsthm}
\usepackage{graphicx}
\usepackage{enumitem}

\usepackage{float,graphicx,epstopdf}
\usepackage{bbm}  
\RequirePackage[reqno]{amsmath}
\usepackage[ruled,vlined]{algorithm2e} 
\SetKwRepeat{Do}{do}{while}

\usepackage{caption, wrapfig,tabularx}
\usepackage[font=small,position=top]{subcaption}

\usepackage{xcolor}
\usepackage{comment}
\usepackage{hyperref}
\usepackage{tcolorbox}
\usepackage{dsfont,amssymb,amsmath, graphicx,enumitem}
\usepackage{booktabs}
\usepackage{amsfonts,dsfont,mathtools, mathrsfs,amsthm,fancyhdr} 
\usepackage[amssymb, thickqspace]{SIunits}

\usepackage{multirow}
\usepackage{bigstrut}

\usepackage{tikz}

\newtheorem{theorem}{Theorem}

\newtheorem{remark}{Remark}

\usepackage{amsmath}

\newcommand{\ee}{\end{equation}}
\newcommand{\bea}{\begin{equation*}\begin{aligned}}
\newcommand{\eea}{\end{aligned}\end{equation*}}

\newcommand{\R}{\mathbb{R}}

\newcommand{\wh}{\widehat}
\newcommand{\mc}{\mathcal}
\newcommand{\mbb}{\mathbb}

\newcommand{\Pnom}{\wh{\mbb P}}
\newcommand{\QQ}{\mbb Q}

\newcommand{\st}{\mathrm{s.t.}}

\newcommand{\PSD}{\mathbb{S}_{+}} 

\newcommand{\Let}{\triangleq}
\newcommand{\opt}{^\star}

\newcommand{\half}{\frac{1}{2}}

\newcommand{\m}{\mu}
\newcommand{\cov}{\Sigma}
\newcommand{\msa}{\wh \m}
\newcommand{\covsa}{\wh \cov}

\title{Distributional Surgery for Language Model Activations}
\author{
  Bao Nguyen\\
  CUHK \\
  \texttt{nbnguyen@se.cuhk.edu.hk} \\\And
  Binh Nguyen \\
  VinUniversity \\
  \texttt{binh.nt2@vinuni.edu.vn} \\\AND
  Duy Nguyen \\
  UNC-Chapel Hill \\
  \texttt{duykng@cs.unc.edu} \\\And
  Viet Anh Nguyen  \\
  CUHK \\
  \texttt{nguyen@se.cuhk.edu.hk} \\}

\begin{document}
\maketitle
\begin{abstract}
Language models, while capable of generating remarkably coherent and seemingly accurate text, can occasionally produce undesirable content, including harmful or toxic outputs. In this paper, we present a new two-stage approach to detect and mitigate undesirable content generations by rectifying activations. First, we train an ensemble of layerwise classifiers to detect undesirable content using activations by minimizing a smooth surrogate of the risk-aware score. Then, for detected undesirable contents, we propose layerwise distributional steering policies that transform the attention heads. These policies are computed through principled semidefinite programming, which aims to minimally perturb the attention distribution while probabilistically guaranteeing the effectiveness of the editions. Empirical evaluations across multiple language models and datasets show that our method outperforms baselines in reducing the generation of undesirable output. 
\end{abstract}

\section{Introduction}

Language models (LMs) have demonstrated a remarkable ability to understand and generate human-like documents \cite{ref:radford2019language,ref:brown2020language,ref:touvron2023llama,ref:touvron2023llama2,ref:jiang2023mistral,dubey2024llama}. However, inspection of their output often reveals undesirable generation, including inaccurate or toxic texts \cite{ref:ji2023survey,ref:rawte2023troubling,ref:xu2024hallucination}. Meanwhile, developing effective strategies to control the generation process of LMs remains a significant challenge \cite{ref:tonmoy2024comprehensive}.

Researchers have proposed numerous methods for controllable text generation in language models \cite{ref:zhang2023survey,ref:li2024pre}. These approaches primarily include model editing and supervised fine-tuning. Both methods, however, require altering the model weights using a subset of text samples, which can result in unstable representations for other text instances \cite{ref:hase2024does}. Additionally, these approaches typically require substantial computational resources.

To address these limitations, \emph{activation intervention} has emerged as a promising alternative for controllable text generation \cite{ref:subramani2022extracting,ref:hernandez2023measuring,ref:li2024inference}. This approach involves altering the model activations responsible for undesirable output during inference. Previous research has identified interpretable directions within the activation space of language models that play a causal role during inference. Studies by \citet{ref:burns2022discovering} and \citet{ref:moschella2023relative} suggest that these directions can be manipulated to adjust model behavior in a controlled manner. This body of work suggests that the internal representations of language models are structured in a way that enables fine-grained control over generated text.

Drawing inspiration from these findings, activation intervention frameworks operate on the premise that the information needed to guide the model toward generating a target sentence is \emph{already encoded within the model}. The hidden information is extracted as latent vectors and then used to steer the generation toward producing desirable outputs. The preliminary success of these activation intervention methods motivates our approach to improving the quality and controllability of language model generation.

\textbf{Problem Statement.} We consider a language model consisting of $L$ layers, each layer has $H$ heads, each head has dimension $d$. For example, for Llama-2, we have $L = 32$, $H = 32$, and $d = 128$. The training dataset is denoted by $\mathcal D = (x_i, y_i^*)_{i = 1, \ldots, N}$, the $i$-th text is denoted by $x_i$, and its ground truth label is $y_i^* \in \{0, 1\}$, where the label 1 (positive) represents the \textit{un}desirable text, and the label 0 (negative) represents the desirable text. Our goal is twofold: (i) detect an undesirable text, and (ii) modify an undesirable text into a desirable text.

The activations for a text $x_i$ at layer $\ell \in \{1, \ldots, L\}$ is denoted by $a_{\ell, i}$. The activation at layer $\ell + 1$ is the output of the operation:
\begin{equation} \label{eq:ff}
    \begin{aligned}
    a_{\ell + 1, i} &= a^{\mathrm{mid}}_{\ell, i} + \mathrm{FFN}(a^{\mathrm{mid}}_{\ell, i}), \\
    a^{\mathrm{mid}}_{\ell, i} &= a_{\ell, i} + \sum_{i=1}^H Q_{\ell h}\textrm{Att}(P_{\ell h} a_{\ell, i}).        
    \end{aligned}
\end{equation}
Here, $P_{\ell h} \in \mathbb{R}^{d \times dH}$ is the projection matrix that maps the output of each layer to the $d$ dimensional head space, $\mathrm{Att}$ is the attention operator~\cite{ref:vaswani2017attention}, $Q_{\ell h} \in \mathbb{R}^{dH \times d}$ is the pull-back matrix and $\mathrm{FFN}$ is the feed-forward layer. Each $a_{\ell, i}$ is a concatenation of headwise activations $a_{\ell h, i}$ for $h = 1, \ldots, H$. Inspired by~\citet{ref:li2024inference}, we aim to perform intervention at \emph{some selected $a_{\ell h, i}$, the activations for head $h$ of layer $\ell$}, if we detect that the activation is from undesirable content.

\noindent\textbf{Contributions.} We contribute an activation intervention method to detect and rectify the undesirable generation of LM. We call our method RADIANT (\textbf{R}isk-\textbf{A}ware \textbf{D}istributional \textbf{I}ntervention Policies for Language Models' \textbf{A}ctivatio\textbf{n}s). RADIANT comprises two components:
\begin{enumerate}[leftmargin=5mm]
    \item A layerwise probe: at each layer, we train a classifier to detect undesirable content from the layer's activations. We train a risk-aware logistic classifier for each head that balances the false positive and false negative rates. Then, we aggregate these headwise classifiers' predictions using a voting mechanism to form a layerwise classifier. We then identify one layer where the probe delivers the most reasonable predictive performance. This optimal classifier serves as the detector of undesirable content.
    \item A collection of headwise interventions: given the optimal layer for the layerwise probe found previously, we find for each head in that layer an optimal headwise intervention policy. We choose a simple linear map for this intervention policy that minimizes the magnitude of editing while delivering sufficient distributional guarantees that the undesirable-predicted activations will be edited into desirable-predicted activations. We show that this linear map can be computed efficiently using semidefinite programming.
\end{enumerate}

\subsection{Related Works}

\textbf{Controllable generation.} Controllable text generation methods aim to alter the outputs of large language models in a desired way. One possible approach is model editing~\cite{ref:wang2023knowledge, ref:zhang2024comprehensive}, which involves modifying the parameters of a model to steer its outputs. For example,~\citet{ref:meng2022locating} involves identifying specific middle-layer feedforward modules that correspond to factual knowledge and then altering these weights to correct or update the information encoded by the model. Other notable methods include fine-tuning techniques such as Supervised Fine-Tuning (SFT,~\citealt{ref:peng2023instruction, ref:gunel2020supervised}) and Reinforcement Learning from Human Feedback (RLHF,~\citealt{ref:ouyang2022training, ref:griffith2013policy}).

\noindent\textbf{Probing.} Probing is a well-established framework to assess the interpretability of neural networks~\cite{ref:alain2016understanding,ref:belinkov2022probing}. Probing techniques have been applied to understand the internal representations of transformer architectures in language models such as BERT and GPT. For example,~\citet{ref:burns2022discovering} proposed an unsupervised probing method that optimizes consistency between positive and negative samples.~\citet{ref:marks2023geometry} computes the mean difference between true and false statements and skews the decision boundary by the inverse of the covariance matrix of the activations.

\noindent\textbf{Activation interventions.} Activation intervention at inference time is an emerging technique for controllable generation \cite{ref:turner2023activation,ref:li2024inference,singhrepresentation,yin2024lofit}. Unlike model editing or fine-tuning techniques, inference time intervention does not require altering the model parameters.~\citet{ref:li2024inference} proposed a headwise intervention method for eliciting truthful generated answers of a language model. They first train linear probes on each head of the language model, then shift the activations with the probe weight direction or the mean difference direction. 

There is a clear distinction between our method and ITI when choosing the location of the classifiers and, hence, the location of the interventions. The ITI method builds different headwise classifiers scattered at \textit{different} layers, and it may suffer from distribution shifts: if an activation is intervened, this leads to shifts in the activation values at all subsequent layers in the network. Thus, classifiers trained at subsequent layers can degrade performance, and interventions at subsequent layers can also degrade. On the contrary, we build a layerwise classifier focusing on all heads in the \textit{same} layer and do not suffer from the distributional shifts of the activations.

The recent paper by \citet{singhrepresentation} is closely related to our work. The authors propose a heuristic intervention rule; then, using empirical estimations of the means and covariances of activations data's distributions of desirable and undesirable text, they calculate a closed-form optimal transport plan between these two empirical distributions, assuming they are standard normal. However, this framework does not take into account the semantics of sentences. Another recent method, called LoFit (Localized Fine-Tuning on LLM Representations~\citealt{yin2024lofit}), also identifies a specific subset of attention heads that are crucial for learning a particular task, but then performs fine-tuning on the intervention vectors at those chosen heads to enhance the model's hidden representations. This results in additional training overhead.

\section{Layerwise Risk-aware Probes} \label{sec:probe}

In the first step, we aim to find a classifier $\mathcal C_{\ell h}: \R^d \to \{0, 1\}$ for each head $h = 1, \dots, H$ at each layer $\ell = 1, \ldots, L$ to classify the activation value $a_{\ell h}$ of desirable and undesirable texts. We propose using a linear logistic classifier, parameterized by a slope parameter $\theta_{\ell h} \in \R^d$ and a bias parameter $\vartheta_{\ell h} \in \R$. The headwise classification rule is 
\begin{equation*}
\begin{aligned}
  \mathcal C_{\ell h}(a_{\ell h}) &= \begin{cases}
    1 & \text{if } \mathrm{sigmoid}(\vartheta_{\ell h} + \theta_{\ell h}^\top a_{\ell h}) \ge 0.5, \\
    0 & \text{otherwise,}
    \end{cases}\\
    &= 
    \begin{cases}
    1 & \text{if } \vartheta_{\ell h} + \theta_{\ell h}^\top a_{\ell h} \ge 0, \\
    0 & \text{if } \vartheta_{\ell h} + \theta_{\ell h}^\top a_{\ell h} < 0.
    \end{cases}      
\end{aligned}
\end{equation*}

The training process of $\mc C_{\ell h}$ must take into account two types of risk: (i) false-negative risk when an undesirable text is not detected, (ii) false-positive risk when a desirable text is classified as undesirable, and is subsequently edited and loses its original semantics. Therefore, a natural candidate for the loss function is a combination of the False Positive Rate (FPR) and the False Negative Rate (FNR).  However, neither FPR nor FNR have smooth functions in optimizing variables. We, hence, resort to smooth surrogates of these two metrics that use the predicted probability of the classifier, similarly to~\citet{ref:benedict2022sigmoidf1}. In detail, we use
\begin{align*}
    &\mathrm{FPR}(\theta_{\ell h},\vartheta_{\ell h}) \\
    &= \frac{1}{N_0}         \sum_{i=1}^N  \mathrm{sigmoid}(\vartheta_{\ell h} + \theta_{\ell h}^\top a_{\ell h, i}) \times (1 - y^*_i), \\
    &\mathrm{FNR}(\theta_{\ell h},\vartheta_{\ell h}) \\
    &= \frac{1}{N_1}\sum_{i=1}^N \left(1 - \mathrm{sigmoid}(\vartheta_{\ell h} + \theta_{\ell h}^\top a_{\ell h, i})\right) \times y^*_i.
\end{align*}
The linear probe training loss is thus
\begin{equation} 
    \label{eq:classifier}
    \min_{\theta_{\ell h} \in \R^d,~\vartheta_{\ell h} \in \R}~\mathrm{FPR}(\theta_{\ell h},\vartheta_{\ell h}) + \alpha \mathrm{FNR}(\theta_{\ell h},\vartheta_{\ell h}), 
\end{equation}
for some positive weight parameters $\alpha$. A higher value of $\alpha$ will emphasize achieving a lower false negative rate, which is critical for detecting undesirable inputs. Problem~\eqref{eq:classifier} has a smoothed surrogate loss that is differentiable and can be solved using a gradient descent algorithm. Finally, we aggregate $\{ \mathcal C_{\ell h} \}_{h = 1, \ldots, H}$ into a single classifier $\mathcal C_{\ell}$ for layer $\ell$ by a simple voting rule
\[
    \mathcal C_{\ell}(a_\ell) = \begin{cases}
        1 &\text{if } \sum_{h = 1}^H \mathcal C_{\ell  h}(a_{\ell h}) \ge \tau, \\
        0 & \text{otherwise,}
    \end{cases}
\]
where $\tau \in [0, H]$ is a tunable threshold. When $\tau = \lfloor H/2 \rfloor$, then $\mathcal C_{\ell}$ becomes the majority voting results of the individual (weak) classifiers $\mathcal C_{\ell h}$. We optimize the hyperparameter $\tau$ to reduce the False Negative Rate (FNR), with a secondary focus on the False Positive Rate (FPR) in cases where the FNR rates are equal. The rationale for this choice is that we believe undesirable content being labeled as desirable is more problematic than other instances.

To conclude this step, we compute the classifier $\mathcal C_\ell$ for layer $\ell = 1, \ldots, L$ by tuning the parameters $\alpha$. The layer whose classifier $\mathcal C_{\ell}$ delivers the highest quality (accuracy or any risk-aware metric) will be the optimal layer to construct the probe. This optimal layer, along with the collection of headwise classifiers, is the final output of this step.

\section{Headwise Interventions with Probabilistic Guarantees} \label{sec:OT-intervene}

We propose a distributional intervention to the activations of the samples predicted as undesirable by the layerwise classifier. In this section, we will focus on constructing a single headwise intervention, and in the next section, we will combine multiple headwise interventions into a layerwise intervention. A headwise intervention is a map $\Delta_{\ell h}: a_{\ell h} \mapsto \hat a_{\ell h}$ that needs to balance multiple criteria: (i) it should be easy to compute and deploy, (ii) it should be effective in converting the undesirable activations to the desirable regions, (iii) it should minimize the magnitude of the intervention to sustain the context of the input. Intuitively, we propose solving an optimization problem with the loss and constraints that fit all the criteria listed. The details are as follows.

To promote (i), we employ a simple linear map $\Delta_{\ell h}( a_{\ell h}) = G_{\ell h} a_{\ell h} + g_{\ell h}$ parametrized by a matrix $G_{\ell h} \in \R^{d \times d}$ and a vector $g_{\ell h} \in \R^d$. This linear map can also be regarded as a pushforward map that transforms the \textit{un}desirable-predicted activations to become desirable-predicted activations. Let us now represent the \textit{un}desirable-predicted activations as a $d$-dimensional random vector $\tilde a_{\ell h}$. Its distribution can be estimated using the training data after identifying the subset $\hat{\mathcal D}_{\ell h}^+$ of training samples that are \emph{predicted undesirable} by $\mathcal C_{\ell h}$, that is, $\hat{\mathcal D}_{\ell h}^+ \Let \{i: \mathcal C_{\ell h}(a_{\ell h, i}) = 1 \}$. The activations of samples in $\hat{\mathcal D}_{\ell h}^+$ lead to an empirical distribution $\Pnom_{\ell h}$. The linear map $\Delta_{\ell h}$ will pushforward the distribution $\Pnom_{\ell h}$ to the new distribution $\QQ_{\ell h} = \Delta_{\ell h} \# \Pnom$.

Using the pushforward distribution $\QQ_{\ell h}$, we can impose criteria (ii) and (iii) above in an intuitive method. To promote (ii), we require that the activations distributed under $\QQ_{\ell h}$ should be classified as desirable by $\mathcal C_{\ell h}$ with high probability. Finally, to promote (iii), we require that the distributions $\QQ_{\ell h}$ and $\Pnom_{\ell h}$ be not too far from each other. Let $\gamma \in (0, 0.5)$ be a small tolerance parameter, and let $\varphi$ be a measure of dissimilarity between probability distributions, we propose to find $\Delta_{\ell h}$ by solving the following stochastic program
\begin{equation} \label{eq:intervention}
    \begin{array}{cl}
        \min & \varphi(\Pnom_{\ell h}, \QQ_{\ell h}) \\
        \st & \QQ_{\ell h}(\tilde a~\text{classified by $\mathcal C_{\ell h}$ as 0}) \ge 1- \gamma,\\
        & \QQ_{\ell h} = \Delta_{\ell h} \# \Pnom_{\ell h}.
    \end{array}
\end{equation}

Problem~\eqref{eq:intervention} is easier to solve in specific circumstances. For example, when we impose that both $\Pnom_{\ell h}$ and $\QQ_{\ell h}$ are Gaussian and when we choose $\varphi$ as a moment-based divergence, then $\Delta_{\ell h}$ can be obtained by solving a convex optimization problem. In the next result, we use $\| \cdot \|_F$ as the Frobenius norm of a matrix, and $\Phi$ as the cumulative distribution function of a standard Gaussian distribution.

\begin{theorem}[Optimal headwise intervention]\label{thm:intervene}
    Suppose that $\Pnom_{\ell h} \sim \mc N(\msa, \covsa)$ and $\QQ_{\ell h} \sim \mc N (\mu, \cov)$ and $\varphi$ admits the form
    \[
    \varphi (\Pnom_{\ell h}, \QQ_{\ell h}) = \| \m - \msa \|_2^2 + \| \cov^\half - \covsa^{\half} \|_F^2.
    \]
    Let $(\m\opt, S\opt, t\opt)$ be the solution of the following semidefinite program
    \begin{equation} \label{eq:sdp}
        \begin{array}{cl}
             \min & \| \m - \msa \|_2^2 + \| S - \covsa^{\half} \|_F^2 \\
             \st & \vartheta_{\ell h} + \theta_{\ell h}^\top \mu + \Phi^{-1}(1-\gamma) t \le 0 \\
             & \| S \theta_{\ell h} \|_2 \le t \\
             & \mu \in \R^d,~S \in \PSD^d,~t \in \R_+.
        \end{array}
    \end{equation}
    Then, by defining $G\opt_{\ell h} = \covsa^{-\half} \big( \covsa^{\half} (S\opt)^2 \covsa^{\half} \big)^\half \covsa^{-\half}$ and $g\opt_{\ell h} = \mu\opt - G\opt_{\ell h}\msa$, a linear map $\Delta_{\ell h}$ that solves~\eqref{eq:intervention} is
    \[
    \Delta_{\ell h}(a_{\ell h}) = G\opt_{\ell h} a_{\ell h} + g\opt_{\ell h}.
    \]
\end{theorem}

\begin{figure}[t]
\centering
\includegraphics[width=0.9\linewidth]{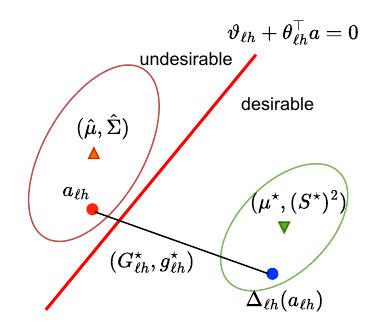}
\vspace{-5mm}
\caption{At head $h$ of layer $\ell$, we learn a headwise intervention, linear mapping $\Delta_{\ell h}$ to transform the \emph{un}desirable-predicted activations to desirable-predicted activations.}
\vspace{-5mm}
\label{fig:head_int}
\end{figure}

The proof of Theorem~\ref{thm:intervene} can be found in Appendix~\ref{sec:proof}. Figure~\ref{fig:head_int} illustrates the effect of the headwise intervention $\Delta_{\ell h}$: The headwise classifier $\mathcal C_{\ell h}$ is represented by the red linear hyperplane $\vartheta_{\ell h} + \theta_{\ell h}^\top a = 0$ on the activation space; the undesirable-predicted (label 1) region is towards the top left corner, while the desirable-predicted (label 0) region is towards the bottom right corner. The activations of the undesirable-predicted samples are represented as a Gaussian distribution with mean $(\msa, \covsa)$, drawn as the red ellipsoid. The edit map $\Delta_{\ell h}$ pushes this distribution to another Gaussian distribution $\QQ_{\ell h}$ drawn as the green ellipsoid. The distribution $\QQ_{\ell h}$ has a coverage guarantee on the desirable-predicted region with probability at least $1-\gamma$. One can also verify that $\QQ_{\ell h}$ has mean $\m\opt$ and covariance matrix $(S\opt)^2$. Problem~\eqref{eq:sdp} can be solved using semidefinite programming solvers such as COPT or Mosek.

The moment information $\msa$ and $\covsa$ can be estimated from the subset $\hat{\mc D}_{\ell h}^+$. One can intuitively expect a trade-off between the tolerance level $\gamma$ and the magnitude of the headwise mapping. If $\gamma$ is lower, the activations will be edited at a higher magnitude so that the edited activations will likely end up in the desirable-predicted region of the classifier $\mathcal C_{\ell h}$. In contrast, if $\gamma$ is higher, the activations will be edited with a smaller magnitude due to the less stringent constraint to swap the predicted label.

One can view the distribution $\QQ_{\ell h} \sim (\m\opt, (S\opt)^2)$ as the counterfactual distribution of the undesirable-predicted activations with \textit{minimal} perturbation. This distribution $\QQ_{\ell h}$ is found by optimization, which is in stark contrast with the design of the counterfactual distribution in MiMic~\cite{singhrepresentation}, in which the intervention is computed based on the activations of the desirable-predicted activations. As a comparison to ITI~\cite{ref:li2024inference}, we note that the headwise intervention of ITI does \textit{not} depend on the value of the activations: ITI shifts the activations along the truthful directions for a stepsize multiplied by the standard deviation of activations along the intervention (truthful) direction. In contrast, our headwise intervention depends on the value $a_{\ell h}$, and one can verify that the magnitude of the proposed shift amounts to $\| (G\opt_{\ell h} - I) a_{\ell h} + g\opt_{\ell h} \|_2$. Moreover, ITI does not provide any (probabilistic) guarantee for the intervention, while the probabilistic guarantee is internalized in our method through the design of the mapping in equation~\eqref{eq:intervention}.

\begin{remark}
\label{remark: exp_trick}
    We observe that the two following tricks increase the empirical performance of our intervention framework. First, to avoid the collapse of $\QQ_{\ell h}$ into a Dirac distribution and to ensure the similarity between the real and the constructed covariance matrix of desirable content, we can add the constraint $S \succeq \widehat\Sigma_0^{\frac{1}{2}}$ to the optimization problem~\eqref{eq:sdp}, where $\widehat\Sigma_0$ is the empirical covariance matrix of the desirable activations $\{i: y_i^* = 0\}$. Second, to avoid taking the inverse cdf of the standard normal distribution, we use $\Gamma \leftarrow \Phi^{-1}(1-\gamma)$ and finetune $\Gamma$ instead of $\gamma$. 
\end{remark}

Finally, given the input with activation $a_{\ell}$ at layer $\ell$, suppose that $a_{\ell}$ is predicted undesirable by $\mathcal C_{\ell}$, we propose to edit the activations of \textit{only} the heads that are predicted undesirable by the headwise classifier $\mathcal C_{\ell h}$. More specifically, we edit the headwise activations $a_{\ell h}$ to a new headwise activations $\hat a_{\ell h}$ through the relationship
\begin{equation*} 
    \hat a_{\ell h} = \begin{cases}  \Delta_{\ell h}(a_{\ell h}) &\text{if } \mathcal C_{\ell h}(a_{\ell h}) \times \mathcal C_{\ell}(a_{\ell}) = 1, \\
    a_{\ell h} & \text{otherwise,}
    \end{cases}
\end{equation*}
where $\Delta_{\ell h}(a_{\ell h}) = G\opt_{\ell h} a_{\ell h} + g\opt_{\ell h}$ for $h = 1, \ldots, H$. In other words, each new headwise activation $\hat a_{\ell h}$ is computed based on three terms: the original headwise activations $a_{\ell h}$, the headwise intervention $\Delta_{h}(a_{\ell h})$, and the indicator value identifying if head $h$ and layer $\ell$ is predicted desirable or undesirable.

\section{Experiments}\label{sec:experiments}

In this section, we present empirical evidence for the effectiveness of our method RADIANT. We evaluate RADIANT on the TruthfulQA benchmark~\citet{ref:lin2021truthfulqa}, consisting of two tasks: the main task is the generation, and the secondary task is multiple choice. The generation task requires the model to generate an entire answer for each question using greedy autoregressive decoding. The accuracy and helpfulness of the answer are best assessed by humans. However, in almost all recent works in the field, including \citet{ref:li2024inference} and \citet{yin2024lofit}, this criterion is measured by an alternative large language model finetuned on the target dataset. The multiple-choice task contains candidate answers to each question, requiring the model to give probabilities. Higher probabilities for truthful answers yield higher scores. 

\subsection{Experimental Settings}
\label{sec:exp_setting}

\noindent \textbf{Datasets.} We evaluate and compare our method with other baselines using the TruthfulQA benchmark~\citet{ref:lin2021truthfulqa}. Details about this dataset and how we preprocess the data can be found in Appendix~\ref{ssec:dataset}. In addition, we also show the generalization of our method by conducting a transferability experiment on two other out-of-distribution datasets, including NQOpen~\cite{ref:kwiatkowski2019natural} and TriviaQA~\cite{joshi2017triviaqa}. Due to space constraints, the results for the latter two datasets are relegated to Appendix~\ref{sec:transfer}.

\noindent \textbf{Hyperparameters.} There are two pivotal hyperparameters in the RADIANT framework, namely $\alpha$ in probe loss~\eqref{eq:classifier}, and $\Gamma = \Phi^{-1}(1-\gamma)$ in the computation of the intervention map~\eqref{eq:sdp}. The discussion about their impact on RADIANT and how to select them is in Appendix~\ref{sec:analysis_hype}.

\noindent \textbf{Baselines.} We benchmark against:
\label{sec: baselines}
\begin{itemize}[leftmargin=5mm]
    \item Inference-time Intervention (ITI)~\cite{ref:li2024inference}, the state-of-the-art method for finetuning-free intervention. The hyperparameters of the baseline follow their original paper and their GitHub repository.\footnote{\url{https://github.com/likenneth/honest_llama/tree/master}}
    \item Few-shot prompting (FSP) introduced in \citet{bai2022training} showcases the effectiveness of 50-shot prompting in benchmark TruthfulQA. 
    \item Probe-Free Low Rank Activation Intervention (FLORAIN)~\cite{jiang2025probe} which trains a low rank component for intervening into one specified hidden representation. 
    \item Instruction Fine-Tuning (IFT, \citealt{wang2022self, chung2024scaling}) is a popular fine-tuning approach to boost the truthfulness of language models. Two notable pretrained models in this direction, namely Alpaca-7B~\cite{ref:taori2023alpaca} and Vicuna-7B~\cite{ref:chiang2023vicuna}, are adopted for comparison.
    \item Representation Intervention Fine-tuning (RIFT) methods aim to adjust language model activations for improved truthfulness. However, they add extra parameters and require extensive computational resources for fine-tuning. We consider LOFiT~\cite{yin2024lofit} for comparison.
    \item Non-Linear Inference Time Intervention (NL-ITI)~\cite{ref:hoscilowicz2024non} extends ITI by introducing a non-linear multi-token probing and multi-token intervention method.
    \item Learnable Intervention for Truthfulness Optimization (LITO)~\cite{ref:bayat2024enhanced} explores a sequence of model generations based on increasing levels of intervention magnitude and then selects the most accurate response.
\end{itemize}

\noindent \textbf{Metrics.} Following the standard benchmark in TruthfulQA~\cite{ref:lin2021truthfulqa, ref:li2024inference}, we use the below metrics:
\begin{itemize}[leftmargin=5mm]
    \item For the multiple choice task, we use MC1 and MC2 metrics as defined in \citet{ref:lin2021truthfulqa}. MC1 measures the model's accuracy in selecting the correct answer from the given choices, where selection is based on the highest log-probability score assigned to each completion. MC2 is the normalized total probability assigned to the set of true answers.
    \item For the generation task, we use two fine-tuned \texttt{GPT-3.5-instruct} models to classify whether an answer is true or false and informative or not. We report two metrics from~\citet{ref:li2024inference}: truthful score True (\%) and True*Info (\%), a product of scalar truthful and informative score. We note that there are discrepancies between the results of ITI reproduced in our work and the original results reported in~\citet{ref:li2024inference}, as the original paper used \texttt{GPT-3} based models to score these two metrics; however, at the time this paper is written, \texttt{GPT-3} is no longer available on the OpenAI platform.
    \item We assess how our method and baselines alter the original generation distribution using two extra metrics: the Kullback-Leibler (KL) divergence and Cross-Entropy (CE) loss of the model's next-token prediction distribution before and after intervention. Due to limited space in the main paper, these metrics for our method and baselines are detailed in Appendix~\ref{sec:kl_ce}.
\end{itemize}

\noindent\textbf{Computing resources.} We run all experiments on 4 NVIDIA RTX A5000 GPUs, an i9 14900K CPU, and 128GB RAM. The semidefinite programs~\eqref{eq:sdp} are solved using Mosek 10.1; the average solving time for each instance is around 50 seconds.

\noindent \textbf{Our repository:} \\
\url{https://github.com/nguyenngocbaocmt02/OT-Intervention}

\subsection{Numerical Results}

\subsubsection{Comparison between Finetuning-free Techniques}

We benchmark two fine-tuning-free baselines (ITI and FSP) along with our framework RADIANT on Llama-7B, Llama3-8B, and Llama2-chat-13B with the TruthfulQA dataset. The results are presented in the \textcolor{red}{\textbf{first three big rows}} of Table~\ref{table:combined-results}. Across the three models, the combined method of FSP + RADIANT consistently achieved the highest scores in metrics such as True * Info and True, with 49\%  for Llama-7B, 44\% for Llama3-8B, and 65\% for Llama2-chat-13B. When running alone, our method, RADIANT, also demonstrated significant improvements, particularly in Llama2-chat-13B, where it achieved a True * Info score of 64\% and a Truthful score of 74\%. This suggests the efficiency of our framework compared with other baselines, including the current state-of-the-art ITI.

\begin{table}[H]
\begin{center}
\fontsize{7.5}{10}\selectfont
\begin{tabular}{p{3mm}p{21mm}p{12mm}p{6mm}p{5mm}p{5mm}}
\toprule
\!\!Model & Methods & True * Info (\%) $\uparrow$ & True (\%) $\uparrow$ & MC1 $\uparrow$ & MC2 $\uparrow$ \\
\midrule
\multirow{10}{*}{\rotatebox{90}{\textcolor{red}{\textbf{Llama-7B}}}}
& Unintervened & 21.15 & 22.16 & 25.58 & 40.54 \\
& ITI & 26.52 & 28.03 & 27.78 & 43.59 \\
& FSP & 36.13 & 39.78 & \textbf{34.03} & \textbf{50.34} \\
& NL-ITI & 29.06 & 38.04 & 32.97 & 45.69 \\
& LITO & 39.08 & 41.22 & 29.22 & 47.64 \\
& FLORAIN & 31.46 & 34.72 & 31.76 & 47.43 \\
& RADIANT (ours) & \textbf{40.36} & \textbf{44.48} & 30.91 & 46.13 \\
\cmidrule{2-6}
& FSP + ITI & 40.63 & 45.16 & 35.50 & 52.48 \\
& FSP + NL-ITI & 45.97 & 47.31 & \textbf{38.37} & 53.61 \\
& FSP + LITO & 49.05 & 55.68 & 36.23 & 54.92 \\
& FSP + FLORAIN & 45.31 & 49.23 & 36.45 & 54.27 \\
& FSP + RADIANT (ours) & \textbf{49.31} & \textbf{57.43} & 37.97 & \textbf{55.31} \\
\midrule
\multirow{10}{*}{\rotatebox{90}{\textcolor{red}{\textbf{Llama3-8B}}}}
& Unintervened & 32.88 & 44.18 & 30.36 & 48.98 \\
& ITI & 35.92 & 46.88 & 32.07 & 49.84 \\
& FSP & 36.32 & 39.78 & \textbf{35.74} & 52.93 \\
& NL-ITI & 35.98 & 45.72 & 33.02 & 51.37 \\
& LITO & 37.53 & 48.20 & 34.96 & 52.54 \\
& FLORAIN & 36.78 & 48.67 & 34.56 & \textbf{53.68} \\
& RADIANT (ours) & \textbf{37.78} & \textbf{50.82} & 33.82 & 52.98 \\
\cmidrule{2-6}
& FSP + ITI & 40.63 & 45.16 & 35.50 & 52.98 \\
& FSP + NL-ITI & 40.70 & 46.03 & 34.15 & 53.35 \\
& FSP + LITO & 43.95 & 49.82 & \textbf{38.41} & 55.31 \\
& FSP + FLORAIN & 42.15 & 47.32 & 36.98 & \textbf{55.83} \\
& FSP + RADIANT (ours) & \textbf{44.09} & \textbf{52.02} & 37.98 & 54.61 \\
\midrule
\multirow{10}{*}{\rotatebox{90}{\textcolor{red}{\textbf{Llama2-chat-13B}}}}
& Unintervened & 51.87 & 59.86 & 35.38 & 53.32 \\
& ITI & 57.02 & 63.04 & 37.46 & 55.59 \\
& FSP & 55.97 & 58.63 & \textbf{40.76} & 57.84 \\
& NL-ITI & 57.13 & 60.82 & 39.01 & 57.24 \\
& LITO & 58.12 & 61.36 & 38.25 & 57.21 \\
& FLORAIN & 60.68 & 67.70 & 39.65 & \textbf{59.57} \\
& RADIANT (ours) & \textbf{63.68} & \textbf{74.20} & 39.95 & 58.18 \\
\cmidrule{2-6}
& FSP + ITI & 56.78 & 59.24 & 41.50 & 59.01 \\
& FSP + NL-ITI & 59.62 & 61.77 & 42.15 & 57.87 \\
& FSP + LITO & 60.74 & 63.21 & 41.28 & 58.46 \\
& FSP + FLORAIN & 61.14 & 62.45 & \textbf{44.52} & \textbf{61.48} \\
& FSP + RADIANT (ours) & \textbf{64.68} & \textbf{67.75} & 42.52 & 59.99 \\
\midrule
\multirow{3}{*}{\rotatebox{90}{\textcolor{blue}{\textbf{Alpaca}}}}
& Base & 30.39 & 30.85 & 26.56 & 41.63 \\
& + ITI & 37.67 & 38.19 & 28.89 & 45.19 \\
& + RADIANT (ours) & \textbf{44.51} & \textbf{45.94} & \textbf{30.79} & \textbf{47.83} \\
\midrule
\multirow{3}{*}{\rotatebox{90}{\textcolor{blue}{\textbf{Vicuna}}}}
& Base & 38.24 & 42.10 & 31.83 & 48.48 \\
& + ITI & 49.27 & 53.25 & 33.42 & 51.80 \\
& + RADIANT (ours) & \textbf{54.87} & \textbf{62.81} & \textbf{35.76} & \textbf{55.14} \\
\midrule
\multirow{9}{*}{\rotatebox{90}{\textcolor{green}{\textbf{Llama variants + LOFiT}}}}
& LOFiT (7B) & 59.48 & 69.03 & 51.04 & 70.78 \\
& + ITI & 60.84 & \textbf{72.29} & 51.41 & 70.84 \\
& + RADIANT (ours) & \textbf{61.50} & 72.08 & \textbf{51.80} & \textbf{71.29} \\
\cmidrule{2-6}
& LOFiT (8B) & 68.80 & 90.08 & 59.00 & \textbf{77.93} \\
& + ITI & 67.57 & 79.31 & 55.33 & 75.85 \\
& + RADIANT (ours) & \textbf{71.47} & \textbf{90.19} & \textbf{59.30} & 76.56 \\
\cmidrule{2-6}
& LOFiT (Chat-13B) & 66.35 & 81.89 & 57.04 & \textbf{76.17} \\
& + ITI & 66.00 & 78.09 & 55.08 & 75.25 \\
& + RADIANT (ours) & \textbf{69.63} & \textbf{83.86} & \textbf{57.45} & 75.47 \\
\bottomrule
\end{tabular}
\end{center}
\caption{Quantitative results of different intervention methods on TruthfulQA dataset, across different Language Models and fine-tuning approaches. Parameters of RADIANT: $\alpha = 2.5, \Gamma = 15$.}
\label{table:combined-results}
\end{table}

\subsubsection{Comparison against ITI with Instruction Finetuning Methods.}

We investigate whether implementing RADIANT on Alpaca and Vicuna, two instruction fine-tuning models from Llama-7B, can further enhance their performances. Results in Table~\ref{table:combined-results} (\textcolor{blue}{\textbf{fourth and fifth big rows}}) indicate that applying RADIANT significantly enhances both the baseline models, with Alpaca + RADIANT improved to 44.5\% in True*Info score and 46\% in Truthful score. Similarly, Vicuna + RADIANT achieved the highest scores of 55\% in True*Info score and 63\% in Truthful score, showcasing a marked increase compared to its baseline performance of 38\% and 42.1\%, respectively. In both cases, RADIANT outperformed ITI, demonstrating its effectiveness in enhancing the models' accuracy and truthfulness.

\begin{table*}[t]
\centering
\begin{small}
\begin{tabular}{lcccc}
\toprule
Methods & True * Info (\%) $\uparrow$ & True (\%) $\uparrow$ & MC1 $\uparrow$ & MC2 $\uparrow$ \\ \midrule
Unintervened                                      & 21.15          & 22.16          & 25.58          & 40.54          \\
ITI                                               & 26.52          & 28.03          & 27.78          & 43.59          \\
1st scenario: Our linear probe + ITI intervention & 26.88          & 28.00          & 29.00          & 44.00          \\
1st scenario: ITI linear probe + our intervention & 36.66          & 39.00          & 28.00          & 43.00          \\
2nd scenario: Cross entropy loss                  & 30.36          & 33.00          & 29.00          & 43.00          \\
RADIANT                                           & \textbf{40.36} & \textbf{44.48} & \textbf{30.91} & \textbf{46.13} \\ \bottomrule
\end{tabular}
\end{small}
\caption{Ablation study results: in the first scenario, we swap heads selected by RADIANT with ITI intervention, and vice versa; in the second scenario, we replace our risk-aware loss function with cross-entropy loss in training linear probe. Performed on TruthfulQA with Llama-7B.}
\label{table:ablation-study}
\end{table*}

\begin{table*}
\centering
\begin{small}
\begin{tabular}{lccc}
\toprule
Component                                  & Llama-7B & Llama3-8B & Llama2-chat-13B \\ \midrule
Train the linear probe for one layer (s)                      & 15.64    & 17.32    & 29.42          \\
Compute intervention for one head (s)      & 52.33    & 58.43    & 55.67          \\
Avg. increase in inference time per answer (\%) & 3.09     & 3.32     & 4.72           \\ \bottomrule
\end{tabular}
\end{small}
\caption{Wall-clock time breakdown by components of RADIANT for different pretrained models.}
\label{table:running-time}
\vspace{-5mm}
\end{table*}

\subsubsection{Comparison against ITI with Representation Intervention Finetuning Methods.}

We apply RADIANT and ITI on Llama-7B, Llama3-8B, and Llama2-chat-13B models, which were previously fine-tuned by LOFiT, a representation intervention finetuning method. The experimental results in the \textcolor{green}{\textbf{last big row}} of Table~\ref{table:combined-results} show that RADIANT is better than ITI in improving correctness and informativeness across different Llama models. While ITI offers modest improvements in some instances, it generally lags behind RADIANT, especially in larger models.

\subsection{Ablation Studies}

We perform two ablation studies to demonstrate the effectiveness of our framework.  Table~\ref{table:ablation-study} reports the performance of the Llama-7B + TruthfulQA dataset. In the first scenario, we select intervened heads using ITI, then compare our intervention approach versus ITI. We noticed that switching the head selection between RADIANT and ITI improved performance when the RADIANT intervention was applied, reaching 37\% in the True * Info score. In the second scenario, the probing loss function is replaced by the popular binary cross-entropy loss. This scenario tests the impact of replacing the risk-aware loss function with cross-entropy loss, which resulted in moderate improvements but still fell short compared to RADIANT's risk-aware loss in Section~\ref{sec:probe} (30.36\% vs 40.36\% in True*Info). Overall, these findings suggest that both the choice of intervention and the loss function play crucial roles in our framework.

\subsection{Computational Cost}

Our method is computationally cheap: for each head, our linear probes require one vector-vector multiplication, and our linear interventions require only one matrix-vector multiplication. To demonstrate this, we clocked the running time to calculate the intervention vectors on an A5000 GPU for the Llama-7B and Llama3-8B models and on two A5000 GPUs for Llama2-chat-13B and show the results in Table~\ref{table:running-time}. Our intervention only slightly increases the running time of the inference process. In addition to its simplicity, the preprocessing of our framework for calculating intervention vectors is much less time-consuming and resource-intensive than fine-tuning methods.

\subsection{Additional Experiments}

We perform multiple additional benchmarks to demonstrate the effectiveness of RADIANT across diverse settings; however, due to space constraints, we have included the rest of these experiments in the Appendix. More specifically:
\begin{itemize}[leftmargin=5mm]
    \item  We conduct a hyperparameter analysis and compare our risk-aware loss function with weighted negative log-likelihood (Section~\ref{sec:additional}). We also report the results for the Kullback-Leibler (KL) and Cross-Entropy (CE) metric for generation (Section~\ref{sec:kl_ce}). We find consistent improvements of RADIANT compared to other competitors. 
    \item We develop a parallel implementation (RADIANT-P) in Section~\ref{ssec:computational_cost} that significantly reduces computational overhead while maintaining the performance of the intervention.
    \item  We conduct a transferability analysis (Section~\ref{sec:transfer}) to demonstrate that RADIANT has better generalization to the NQOpen and TriviaQA datasets than ITI. 
    \item We test RADIANT on various model architectures, including Gemma and GPT models (Section~\ref{sec:other_models}), Mistral and Qwen models (Section~\ref{ssec:mistral-qwen-benchmark}), and sparse MoE architectures (Section~\ref{ssec:sparse-moe}). We also compare RADIANT with supervised fine-tuning (Section~\ref{ssec:radiant-vs-sft}). 
    \item We evaluate RADIANT on other NLP tasks, including toxicity mitigation (Section~\ref{sec:tox_exp_setting}), long fact generation (Section~\ref{sec:long-fact}), and creative writing (Section~\ref{sec:creative_writing}). RADIANT shows comparable performance to computationally intensive methods while maintaining better fluency.
\end{itemize}

\section{Conclusion} \label{sec:conclusion}

We introduced RADIANT, an activation intervention method consisting of two components: (i) a layerwise probe to detect undesirable content and (ii) headwise interventions to rectify the head activations upon undesirably predicted outcomes. Unlike existing intervention methods, which can be scattered across different layers, our intervention is focused on a single layer of the network. This focus helps alleviate the distributional shifts of the activations in subsequent layers.
Moreover, our headwise intervention aims to minimize perturbations to the activations while maintaining a reasonable guarantee of the intervention's effectiveness. This is further demonstrated in empirical results, where our method outperforms the baseline intervention methods for various LMs.

\newpage
\noindent\textbf{Limitations and Social Impact.} Our paper focuses on improving the truthfulness of LMs, and the results aim to improve trustworthy artificial intelligence. Apart from language generation, our paper can also be implemented in other domains for activation editing. However, it is important to acknowledge the potential misuse of our method. There exists a risk that adversarial actors could exploit our approach to transform truthful outputs into misleading or false information. This dual-use nature underscores the importance of ethical guidelines and safeguards in developing artificial intelligence. By promoting transparency and accountability in the use of our framework, we aim to raise awareness of the risks while maximizing the benefits of enhanced truthfulness in language generation.

\bibliography{acl.bbl}

\newpage
\appendix

\section{Additional Results on TruthfulQA Benchmark} 

\subsection{Information about TruthfulQA Dataset}\label{ssec:dataset}

The TruthfulQA dataset is a Question-Answer dataset containing 817 questions that likely elicit false answers from humans due to common misconceptions. We follow the same data-processing used in \citet{ref:li2024inference} and \citet{yin2024lofit} that splits the dataset into train/validation/test with the rate of $326/82/407$ questions and utilizes two-fold cross-validation. Each question has an average length of nine words and has two sets of desirable and undesirable answers. Following \citet{ref:li2024inference}, we separate the original dataset into 5918 question-answer pairs; each has a binary label, indicating desirability. Only pairs associated with questions in the training dataset are used to create our intervention policy, while those in the validation test are set aside for parameter tuning.

\subsection{Additional Ablation Studies} \label{sec:additional}

\subsubsection{Layer Selection Threshold with the Smooth Probing Loss}

\begin{figure}[h]
    \centering
    \begin{subfigure}{\linewidth}
        \centering
        \includegraphics[width=\linewidth]{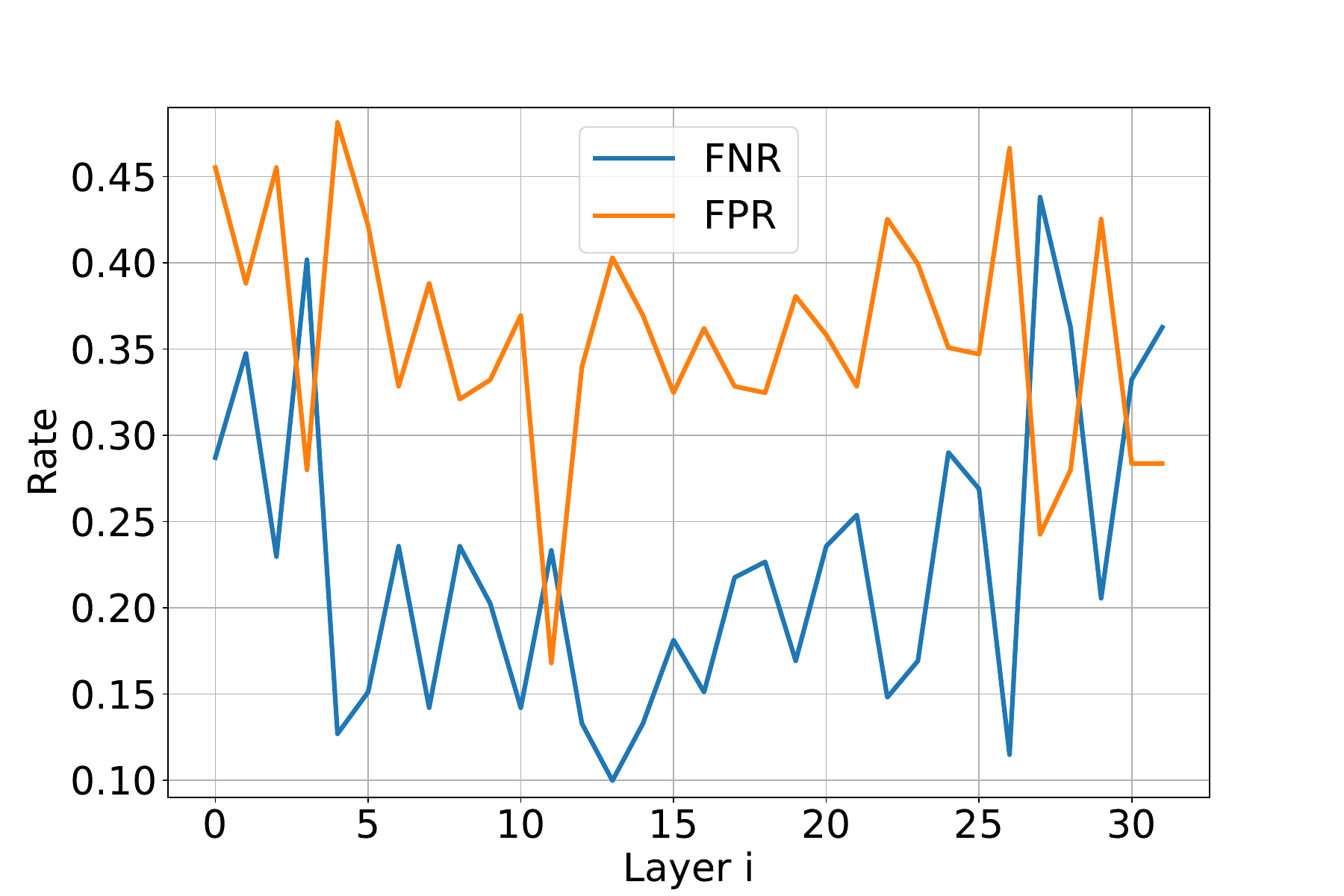}
        \caption{False Negative Rate (FNR) and False Positive Rate (FPR) across layers for intervention threshold $\tau = 11$.}
        \label{fig:fnr_fpr}
    \end{subfigure}
    \hfill
    \begin{subfigure}{\linewidth}
        \centering
        \includegraphics[width=\linewidth]{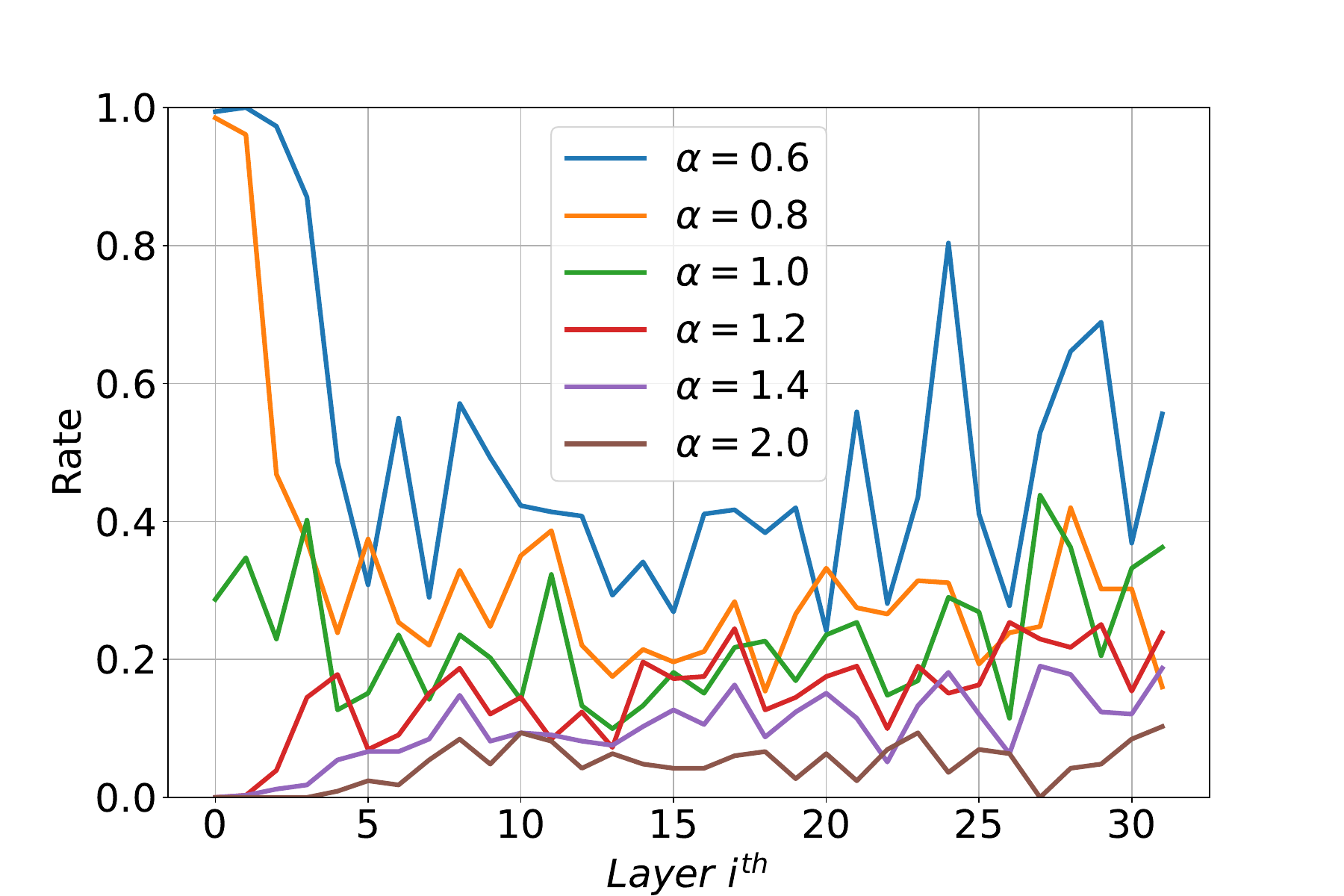}
        \caption{FNR across layers for different value of regularization parameter $\alpha$ of the risk-aware loss in~\eqref{eq:classifier}.}
        \label{fig:alpha_fnr}
    \end{subfigure}
    \caption{FNR and FPR metrics with different hyperparameters $\alpha$ across layers of Llama-7B.} 
    \label{fig:twosub}
\end{figure}

Figure~\ref{fig:twosub} presents the FNR and FPR results for the layerwise probes on Llama-7B on the TruthfulQA dataset. From Figure~\ref{fig:fnr_fpr}, one observes that the optimal layer tends to be a mid-layer ($\ell$ between 11 and 14) with smaller FNR and FPR values. Figure~\ref{fig:alpha_fnr} shows that increasing $\alpha$ will dampen the FNR rate across layers.

\subsubsection{Loss Function of the Classifier}\label{ssec:loss-function-comparison}

To evaluate the effectiveness of our risk-aware loss function in~\eqref{eq:classifier}, we conduct an ablation study that compares it with weighted negative log-likelihood (NLL) loss. The weighted NLL loss to train the classifier is defined as:
\begin{align*}
&\mathcal{L}_{\mathrm{NLL}}(\theta_{\ell h}, \vartheta_{\ell h}) \\&= - \sum_{i=1}^N [w_1 y_i^*  \log(p_i) 
+ w_0 (1 - y_i^*) \log(1 - p_i)],
\end{align*}
where $p_i = \mathrm{sigmoid}(\vartheta_{\ell h} + \theta_{\ell h}^\top a_{\ell h, i})$ is the predicted probability for sample $i$, and $w_1$ and $w_0$ represent the weights applied to positive and negative samples, respectively. These weights are chosen to align with the class-imbalance handling and error-prioritization approach in our proposed loss. Specifically, we set 
\begin{equation*}
w_0 = \frac{1}{N_0}, \quad w_1 = \frac{\alpha}{N_1}
\end{equation*}
to be consistent with the definitions of FPR and FNR in our formulation. We follow the same procedure to choose $\alpha$ as we do with RADIANT and run the experiment for three models: Llama-7B and Llama3-8B on the TruthfulQA dataset. The experimental setup is the same as in~\ref{sec:exp_setting}. The layers for intervention found by our framework with the NLL loss coincide with the layer found by our proposed loss~\eqref{eq:classifier} across all three models. We report the experimental results in Table~\ref{tab:nll_comparison_8b}.

\begin{table}[h]
\centering
\begin{small}
\begin{tabular}{p{1.4cm}p{1.4cm}p{1.2cm}p{1cm}p{1cm}}
\toprule
Methods &\textbf{True*Info (\%) $\uparrow$} & \textbf{True (\%) $\uparrow$} & \textbf{MC1 $\uparrow$} & \textbf{MC2 $\uparrow$} \\ \midrule
Unintervened & 21.15 & 22.16 & 25.58 & 40.54 \\
RADIANT (NLL) & 37.64 & 43.78 & \textbf{31.98} & 46.11 \\
RADIANT & \textbf{40.36} & \textbf{44.48} & 30.91 & \textbf{46.13} \\ \bottomrule
\end{tabular}
\end{small}
\caption{Results for Llama-7B.}
\label{tab:nll_comparison_7b}
\end{table}

\begin{table}[h]
\centering
\begin{small}
\begin{tabular}{p{1.4cm}p{1.4cm}p{1.2cm}p{1cm}p{1cm}}
\toprule
\textbf{Methods} & \textbf{True*Info (\%) $\uparrow$} & \textbf{True (\%) $\uparrow$} & \textbf{MC1 $\uparrow$} & \textbf{MC2 $\uparrow$} \\ \midrule
Unintervened & 32.88 & 44.18 & 30.36 & 48.98 \\
RADIANT (NLL) & 38.36 & 47.25 & 32.12 & 49.83 \\
RADIANT & \textbf{37.78} & \textbf{50.82} & \textbf{33.82} & \textbf{52.98} \\ \bottomrule
\end{tabular}
\end{small}
\caption{Results for Llama3-8B.}
\label{tab:nll_comparison_8b}
\end{table}

Our results show that, while the weighted NLL provides a strong baseline, RADIANT consistently outperforms it in most metrics and models. In particular, RADIANT improves performance in MC1 and MC2, demonstrating a stronger ability to guide model interventions. Although weighted NLL is a standard formulation, our findings suggest that RADIANT's approach to error weighting provides additional benefits beyond simple likelihood-based weighting.

We provide theoretical intuition for why the RADIANT loss might lead to superior performance in practical classification tasks compared to weighted NLL in our problem. The weighted NLL measures the classification confidence via the negative logarithm of predicted probabilities. Thus, large mispredictions incur exponentially large penalties:
\begin{equation*}
-\log(p_i) \xrightarrow[p_i\to0]{} \infty,\quad -\log(1-p_i)\xrightarrow[p_i\to1]{}\infty.
\end{equation*}
This exponential penalty will ensure that the boundary does not go deeper into the area of any classes because if that happens, many samples will have $p_i \to 1$, then the exponential penalty will pull it back. In our problem, the false negative samples are critical, so we expect the boundary to go deeper into the desirable area, and this loss is not suitable.
\begin{table}[h]
\centering
\begin{small}
\begin{tabular}{lcc}
\toprule
\textbf{Methods} & \textbf{Llama 7B} & \textbf{Llama3-8B} \\ \midrule
RADIANT(NLL) & 0.104 & 0.128 \\ 
RADIANT & 0.023 & 0.054 \\ \bottomrule
\end{tabular}
\end{small}
\caption{False Negative Rate comparison between RADIANT loss and Weighted NLL loss denoted as RADIANT(NLL).}
\label{tab:fnr_comparison}
\end{table}

\begin{table*}[t]
\centering
\begin{small}
\begin{tabular}{lccccccc}
\toprule
$\Gamma$ &
  True * Info (\%) $\uparrow$ &
  True (\%) $\uparrow$ &
  Info (\%) $\uparrow$ &
  \multicolumn{1}{c}{MC1 $\uparrow$} &
  \multicolumn{1}{c}{MC2 $\uparrow$} &
  \multicolumn{1}{c}{CE $\downarrow$} &
  \multicolumn{1}{c}{KL $\downarrow$} \\ \midrule
Unintervened & 21.15 & 22.16 & 95.47 & 25.58 & 40.54 & 2.13 & 0.00 \\
5            & 26.14 & 28.40 & 92.04 & 26.81 & 41.91 & 2.14 & 0.01 \\
10           & 33.04 & 36.11 & 91.49 & 27.17 & 43.11 & 2.17 & 0.04 \\
15           & 40.36 & 44.48 & 90.75 & 30.91 & 46.13 & 2.19 & 0.07 \\
20           & 36.59 & 43.46 & 84.20 & 28.15 & 44.92 & 2.29 & 0.18 \\ \bottomrule
\end{tabular}
\end{small}
\caption{The performance of RADIANT when varying $\Gamma$ and fixing $\alpha$ of 2.5.}
\label{tab:gamma}
\end{table*}

\begin{table*}[t]
\centering
\begin{small}
\begin{tabular}{lccccccc}
\toprule
$\alpha$ &
  True * Info (\%) $\uparrow$ &
  True (\%) $\uparrow$ &
  Info (\%) $\uparrow$ &
  $\overline{\text{FPR}}$ $\downarrow$ &
  $\overline{\text{FNR}}$ $\downarrow$ &
  CE $\downarrow$ &
  KL $\downarrow$ \\ \midrule
Unintervened & 21.15 & 22.16 & 95.47 & - & - & 2.13 & 0.00 \\
1.0          & 24.39 & 25.95 & 94.00 & 0.32              & 0.32              & 2.14 & 0.01 \\
1.5          & 29.07 & 31.95 & 91.00 & 0.67              & 0.11              & 2.18 & 0.05 \\
2.0          & 34.75 & 39.54 & 91.88 & 0.76              & 0.05              & 2.19 & 0.06 \\
2.5          & 40.36 & 44.48 & 90.75 & 0.78              & 0.00              & 2.19 & 0.07 \\
3.0          & 34.21 & 38.92 & 87.88 & 0.97              & 0.00              & 2.20 & 0.13 \\ \bottomrule
\end{tabular}
\end{small}
\caption{The performance of RADIANT when varying $\alpha$ and fixing $\Gamma$ of 15.}
\label{tab:alpha}
\end{table*}

Instead, RADIANT measures risk more directly (FPR, FNR), penalizing misclassification linearly in probability. Thus, while weighted NLL severely penalizes misprediction (logarithmic scale), RADIANT loss penalizes misclassification risk proportionally and is, therefore, more reasonable for our problem. Moreover, the exponential penalty potentially makes weighted NLL brittle or overly sensitive to outliers, which often exist in deep layers' activations of complex architectures like transformers. This intuition is backed up by our empirical observation that the False Negative of the weighted NLL loss is consistently higher than our loss across two models, Llama 7B and Llama3-8B.

\subsubsection{The Effect of $\Gamma$ and $\alpha$ on the Performance of RADIANT}
\label{sec:analysis_hype}
The hyperparameter $\alpha$ controls the conservativeness of the classifier in terms of the False Negative Rate. High values of $\alpha$ ensure that no undesirable content goes undetected. However, excessively large values of $\alpha$ may lead to trivial classifiers that classify all samples as undesirable. Such classifiers can be identified by checking if their False Positive Rate in the validation set is one. Therefore, for a given $\alpha$, along with other performance metrics, we report the average False Positive Rate and the average False Negative Rate across all trained classifiers on the validation set denoted $\overline{\text{FPR}}$ and $\overline{\text{FNR}}$.

In Table~\ref{tab:alpha}, we present metrics on the validation set while varying $\alpha$ within the set $\{1.0, 1.5, 2.0, 2.5, 3.0\}$. We use the base model Llama-7B. RADIANT's performance improves as $\alpha$ increases until a significant drop occurs when trivial classifiers dominate at $\alpha = 3.0$. This observation supports our approach of selecting $\alpha$ as high as possible without encountering the trivial classifiers issue. However, the information score decreases as $\alpha$ increases. This decrease can be attributed to RADIANT becoming more conservative and avoiding providing uncertain information. In practice, depending on the information sensitivity of the application of LMs, we can select $\alpha$ as a trade-off between the accuracy of the information and the informativeness. For example, LMs in the medical or legal sectors should avoid providing incorrect or uncertain information, so high values of $\alpha$ are recommended.

We report the performance metrics of Llama-7B when varying $\Gamma$ in Table~\ref{tab:gamma}. This hyperparameter decides how much RADIANT post-intervention activations deviate from the original ones if detected as undesirable. We observe that the True score of RADIANT increases in $\Gamma$. This is because the increasing value of $\Gamma$ drives activations to reside deeper inside the desirable area, thus increasing the probability of desirable generation. However, the larger value of $\Gamma$ makes activations move farther from the original value, as shown by the increase in the CE and KL metrics. The extreme deviation from the original activations leads to inconsistency in semantics. It creates more non-natural sentences, which can be observed at $\Gamma = 20$ with the drop in the Information score. Therefore, a reasonable score should balance the True and Information scores.

In our implementation, for each pre-trained model, we perform a grid search where $\alpha$ ranges over $\{1.0, 1.5, 2.0, 2.5\}$ and $\Gamma$ over $\{5, 7.5, 10, 15, 20\}$ to select the optimal combination based on the True * Info score in the validation set. After running RADIANT with various pre-trained models, we find that the combination of $\Gamma = 15$ and $\alpha = 2.5$ performs effectively in most cases. Unless otherwise specified, we utilize these values for our experiments.


\subsection{RADIANT Enhances Performance with Minimal Distribution Shift}
\label{sec:kl_ce}

Since headwise intervention applies a linear mapping to transform undesirable activations into desirable ones, it raises concerns about potential unintended shifts in meaning. To mitigate the risk of semantic drift, our intervention method is explicitly constrained via semidefinite programming to minimize the shift in activation space while enforcing a probabilistic guarantee that the modified activation crosses the classifier's decision boundary.

To showcase that the meaning shift caused by RADIANT is minimal, we report two additional metrics: Kullback-Leibler (KL) divergence of the model’s next-token prediction distribution (pre- vs. post-intervention) and Cross-Entropy (CE) loss. These metrics quantify the shift in the generation distribution following the intervention. Lower values indicate smaller deviation from the original model’s behavior, reducing the likelihood of unnatural outputs or anomalous characters. The calculation details are provided in \citet{ref:li2024inference}. Due to space constraints, these metrics were omitted from the main paper. 

We report the KL and CE values in Table~\ref{table:kl-ce}, which includes all the settings reported in Table~\ref{table:combined-results} from the main text. Our results show that RADIANT maintains comparable KL and CE values across various scenarios, demonstrating that it preserves the original distribution while significantly improving truthfulness.

\begin{table}[!h]
\begin{center}
\fontsize{7.5}{10}\selectfont
\begin{tabular}{llcc}
\toprule
Model & Methods & CE $\downarrow$ & KL $\downarrow$ \\
\midrule
\multirow{10}{*}{\rotatebox{90}{\textcolor{red}{\textbf{Llama-7B}}}}
& Unintervened & 2.13 & 0.00 \\
& ITI & 2.20 & 0.07 \\
& FSP & 2.13 & 0.00 \\
& NL-ITI & 2.19 & 0.07 \\
& LITO & 2.19 & 0.07 \\
& RADIANT (ours) & 2.19 & 0.07 \\
\cmidrule{2-4}
& FSP + ITI & 2.20 & 0.07 \\
& FSP + NL-ITI & 2.20 & 0.07 \\
& FSP + LITO & 2.20 & 0.07 \\
& FSP + RADIANT (ours) & 2.20 & 0.08 \\
\midrule
\multirow{10}{*}{\rotatebox{90}{\textcolor{red}{\textbf{Llama3-8B}}}}
& Unintervened & 2.38 & 0.00 \\
& ITI & 2.50 & 0.13 \\
& FSP & 2.38 & 0.00 \\
& NL-ITI & 2.50 & 0.13 \\
& LITO & 2.48 & 0.11 \\
& RADIANT (ours) & 2.48 & 0.08 \\
\cmidrule{2-4}
& FSP + ITI & 2.48 & 0.14 \\
& FSP + NL-ITI & 2.49 & 0.14 \\
& FSP + LITO & 2.54 & 0.17 \\
& FSP + RADIANT (ours) & 2.52 & 0.15 \\
\midrule
\multirow{10}{*}{\rotatebox{90}{\textcolor{red}{\textbf{Llama2-chat-13B}}}}
& Unintervened & 2.31 & 0.00 \\
& ITI & 2.32 & 0.17 \\
& FSP & 2.31 & 0.00 \\
& NL-ITI & 2.33 & 0.17 \\
& LITO & 2.34 & 0.18 \\
& RADIANT (ours) & 2.35 & 0.18 \\
\cmidrule{2-4}
& FSP + ITI & 2.33 & 0.13 \\
& FSP + NL-ITI & 2.34 & 0.15 \\
& FSP + LITO & 2.36 & 0.17 \\
& FSP + RADIANT (ours) & 2.38 & 0.18 \\
\midrule
\multirow{3}{*}{\rotatebox{90}{\textcolor{blue}{\textbf{Alpaca}}}}
& Base & 2.81 & 0.00 \\
& + ITI & 2.88 & 0.14 \\
& + RADIANT (ours) & 2.81 & 0.13 \\
\midrule
\multirow{3}{*}{\rotatebox{90}{\textcolor{blue}{\textbf{Vicuna}}}}
& Base & 2.67 & 0.00 \\
& + ITI & 2.77 & 0.26 \\
& + RADIANT (ours) & 2.73 & 0.27 \\
\midrule
\multirow{9}{*}{\rotatebox{90}{\textcolor{green}{\textbf{Llama variants + LOFiT}}}}
& LOFiT (7B) & 2.35 & 0.00 \\
& + ITI & 2.55 & 0.14 \\
& + RADIANT (ours) & 2.56 & 0.13 \\
\cmidrule{2-4}
& LOFiT (8B) & 3.27 & 0.00 \\
& + ITI & 3.33 & 0.08 \\
& + RADIANT (ours) & 3.38 & 0.11 \\
\cmidrule{2-4}
& LOFiT (Chat-13B) & 2.52 & 0.00 \\
& + ITI & 2.73 & 0.21 \\
& + RADIANT (ours) & 2.73 & 0.20 \\
\bottomrule
\end{tabular}
\end{center}
\caption{Quantitative results of different intervention methods on the TruthfulQA dataset, across different Language Models and fine-tuning approaches. Parameters of RADIANT: $\alpha = 2.5, \Gamma = 15$.}
\label{table:kl-ce}
\end{table}

\subsection{Computational Cost -- Paralled Version}\label{ssec:computational_cost}

This section studies the impact of ITI and RADIANT on the base models' inference speed. From the theoretical aspect, it is evident that a head intervention of ITI, which is just a vector addition, is faster than that of RADIANT, which comprises a matrix multiplication and addition operator. This observation is proved again by the empirical results shown in Table~\ref{exp:add_cost}. This table reports the average percentage increase in inference time per answer of ITI and RADIANT across the base models. It is observed that the normal version of RADIANT imposes more additional time in inference than ITI does. However, it should be noted that all RADIANT interventions are conducted on the same layer, while ITI interventions are carried out on multiple pairs of layer heads. This attribute of RADIANT allows us to parallel the interventions, which is impossible for ITI. We denote the parallel version of RADIANT as RADIANT-P and include it in Table~\ref{exp:add_cost}. RADIANT-P offers the same decent results as RADIANT but imposes less computation cost to base models than RADIANT and ITI. 

\begin{table*}[h]
\centering
\begin{small}
\begin{tabular}{lccc}
\toprule
Base models     & ITI  & RADIANT & RADIANT-P \\ \midrule
Gemma-2B        & 2.53 & 6.82    & 1.75      \\
GPT-2 Large     & 2.43 & 3.01    & 1.65      \\
Llama-7B        & 2.46 & 3.09    & 1.45      \\
Llama3-8B       & 2.51 & 3.32    & 1.55      \\
Llama2-chat-13B & 2.51 & 4.72    & 1.57      \\ \bottomrule
\end{tabular}
\end{small}
\caption{The average percentage increase in inference time per answer of ITI and RADIANT across base models.}
\label{exp:add_cost}
\end{table*}

\begin{table*}
\centering
\begin{small}
\begin{tabular}{llcccccc}
\toprule
Dataset &
  Methods &
  True * Info (\%) $\uparrow$ &
  True (\%) $\uparrow$ &
  \multicolumn{1}{c}{MC1 $\uparrow$} &
  \multicolumn{1}{c}{MC2 $\uparrow$} &
  \multicolumn{1}{c}{CE $\downarrow$} &
  \multicolumn{1}{c}{KL $\downarrow$} \\ \hline
\multirow{3}{*}{NQOpen}   & Unintervened   & 17.16          & 18.50          & 40.90          & 53.10          & 2.13 & 0.00 \\
                          & ITI            & 16.97          & 18.90          & 40.40          & 52.94          & 2.20 & 0.07 \\
                          & RADIANT (ours) & \textbf{20.66} & \textbf{22.10} & \textbf{41.50} & \textbf{54.38} & 2.16 & 0.04 \\ \hline
\multirow{3}{*}{TriviaQA} & Unintervened   & 87.82          & 92.25          & 32.60          & 64.35          & 2.13 & 0.00 \\
                          & ITI            & 91.14          & 94.20          & 32.70          & 65.16          & 2.21 & 0.09 \\
                          & RADIANT (ours) & \textbf{92.35} & \textbf{96.50} & \textbf{35.30} & \textbf{67.20} & 2.23 & 0.09 \\ \bottomrule
\end{tabular}
\end{small}
\caption{Quantitative results of the transferability of RADIANT's intervention on different datasets.}
\label{table:transferability}
\end{table*} 

\subsection{Transferability Experiments} \label{sec:transfer}

We train the intervention mappings for  Llama-7B  using the TruthfulQA dataset, but we evaluate its performance on the NQOpen dataset~\cite{ref:kwiatkowski2019natural}. The NQOpen dataset contains approximately 3600 question-answer pairs. Our intervention vectors show strong performance on the NQOpen dataset, shown in Table~\ref{table:transferability}. This effectiveness is also observed with ITI, as noted in its original paper. Nevertheless, our experiment indicates that our intervention mappings offer superior transferability and generality compared to ITI's. This experiment demonstrates RADIANT's effectiveness and highlights the generality of the computed intervention for NLP tasks.

\section{Additional Comparison Results}

\subsection{Evaluation on GPT and Gemma Base Models}
\label{sec:other_models}

In this experiment, we study the performance of finetuning-free techniques, including ITI, RADIANT, and FSP, on Gemma-2B~\cite{ref:team2024gemma} and GPT-2 Large~\cite{ref:radford2019language}, which serve as alternative base models to the Llama model family. Table~\ref{table:benchmark-others} shows that RADIANT, using few-shot prompting, outperforms other methods by a large gap. In particular, FSP + RADIANT improves the True * Info score of Gemma-2B and GPT-2 Large by 25.14\% and 16.16\%, respectively. Notably, FSP + RADIANT is superior to FSP + ITI in both True * Info and True and MC1 scores. Concurrently, RADIANT, implemented separately, outperforms ITI and FSP in terms of True * Info and True scores, while only slightly behind in MC1 and MC2.

\begin{table*}[h]
\begin{small}
\begin{subtable}[h]{\textwidth}
\centering
\begin{tabular}{lcccccc}
\toprule
Methods & True * Info (\%) $\uparrow$ & True (\%) $\uparrow$ & MC1 $\uparrow$ & MC2 $\uparrow$ & \multicolumn{1}{l}{CE $\downarrow$} & KL $\downarrow$ \\ \midrule
Unintervened  & 31.00 & 51.23 & 27.12 & 43.62 & 2.55 & 0.00 \\
ITI           & 33.42 & 54.74 & 29.14 & 46.01 & 2.64 & 0.17 \\
FSP           & 34.92 & 42.23 & \textbf{35.10} & \textbf{49.24} & 2.55 & 0.0  \\
RADIANT(ours)       & \textbf{35.62} & \textbf{59.62} & 30.34 & 48.06 & 2.62 & 0.15 \\ \midrule
FSP + ITI     & 48.83 & 61.57 & 38.27 & 54.73 & 2.69 & 0.16 \\
FSP + RADIANT(ours) & \textbf{56.14} & \textbf{64.71} & \textbf{39.54} & \textbf{56.98} & 2.65 & 0.09 \\ \bottomrule
\end{tabular}
\caption{Gemma-2B}
\end{subtable}

\begin{subtable}[h]{\textwidth}
\centering
\begin{tabular}{lcccccc}
\toprule
Methods & True * Info (\%) $\uparrow$ & True (\%) $\uparrow$ & MC1 $\uparrow$ & MC2 $\uparrow$ & \multicolumn{1}{l}{CE $\downarrow$} & KL $\downarrow$ \\ \midrule
Unintervened         & 19.20           & 21.91          & 23.57          & 40.75          & 2.8  & 0.0  \\
SFT         & \textbf{35.16}           & 38.28          & \textbf{35.70}          & \textbf{53.57}          & 3.27  & 0.46  \\
ITI                  & 26.94          & 31.09          & 24.68          & 42.31 & 2.94 & 0.13 \\
FSP                  & 21.82          & 27.30          & 25.34 & 42.07          & 2.8  & 0.0  \\
RADIANT (ours)       & 30.18 & \textbf{38.73} & 25.14          & 42.14          & 2.92 & 0.12 \\ \midrule
FSP + ITI            & 29.53          & 30.45          & 25.12          & \textbf{44.79} & 2.98 & 0.18 \\
FSP + RADIANT (ours) & \textbf{35.36} & \textbf{40.41} & \textbf{26.18} & 44.29          & 2.94 & 0.16 \\ \bottomrule
\end{tabular}
\caption{GPT-2 Large}
\end{subtable}
\end{small}
\caption{Quantitative results of different intervention methods on TruthfulQA dataset, across different language models. Parameters of RADIANT: $\alpha = 2.5, \Gamma = 15$.}
\label{table:benchmark-others}
\end{table*}

\subsection{Evaluation on Mistral and Qwen Base Models}\label{ssec:mistral-qwen-benchmark}

RADIANT has already been evaluated on base models beyond Llama, including Gemma-2B and GPT-2 Large (Appendix~\ref{sec:other_models}). 
Because RADIANT relies solely on activation values (rather than logits or internal weights), it is directly applicable to any transformer-based model. 
To further validate this, we conducted an experiment on Mistral-7B-Instruct-v0.2 and Qwen2-7B-Instruct base models and evaluated the results using GPT4. The optimal intervention layers found by RADIANT are 18 and 25, respectively, for Qwen2-7B-Instruct and Mistral-7B-Instruct-v0.2.

\begin{table*}[h]
\centering
\begin{small}
\begin{tabular}{p{3cm}p{2.5cm}p{2cm}p{2cm}p{2cm}}
\toprule
\textbf{Methods} & \textbf{True * Info (\%) $\uparrow$} & \textbf{True (\%) $\uparrow$} & \textbf{MC1 $\uparrow$} & \textbf{MC2 $\uparrow$} \\ \midrule
Unintervened & 80.05 & 67.20 & \textbf{55.69} & 70.93 \\
ITI & 76.87 & 62.30 & 54.23 & 70.41 \\
RADIANT (ours) & \textbf{81.94} & \textbf{68.17} & 55.23 & \textbf{71.75} \\ \bottomrule
\end{tabular}
\end{small}
\caption{Experimental results of baselines for Mistral-7B-Instruct-v0.2.}
\label{tab:mistral_results}
\vspace{-5mm}
\end{table*}

\begin{table*}[h]
\centering
\begin{small}
\begin{tabular}{p{3cm}p{2.5cm}p{2cm}p{2cm}p{2cm}}
\toprule
\textbf{Methods} & \textbf{True * Info (\%) $\uparrow$} & \textbf{True (\%) $\uparrow$} & \textbf{MC1 $\uparrow$} & \textbf{MC2 $\uparrow$} \\ \midrule
Unintervened & 64.13 & 50.55 & 40.02 & 59.89 \\
ITI & 73.68 & 50.30 & \textbf{40.88} & 60.86 \\
RADIANT (ours) & \textbf{77.51} & \textbf{53.05} & 40.39 & \textbf{61.17} \\ \bottomrule
\end{tabular}
\end{small}
\caption{Experimental results of baselines for Qwen2-7B-Instruct.}
\label{tab:qwen_results}
\end{table*}

Tables~\ref{tab:mistral_results} and~\ref{tab:qwen_results} show that RADIANT improves truthfulness and overall performance across different models. For Mistral-7B-Instruct-v0.2, RADIANT achieves the highest True * Info (81.94\%) and True (68.17\%) scores, outperforming both the baseline and ITI. For Qwen2-7B-Instruct, RADIANT leads in True * Info (77.51\%) and True (53.05\%).

Regarding comparisons with knowledge editing methods such as WISE~\cite{wang2024wise} and ConceptEdit~\cite{wang2023editing}, we emphasize that our work targets a different setting: inference-time intervention rather than parameter-based model editing. While model editing methods aim to alter stored knowledge via weight updates, RADIANT operates at inference time by activation intervention, enabling lightweight, reversible control over generation. 

\subsection{Evaluation on Sparse MoE Architectures}\label{ssec:sparse-moe}

We show that RADIANT can be applied to a sparse Mixture of Experts. The core mechanism of RADIANT, layerwise and headwise activation editing, relies solely on access to activations at a given layer and does not depend on weight sharing or architectural uniformity. Moreover, since our method operates directly on attention head activations, and most sparse MoE designs apply sparsity to MLP layers rather than attention blocks, the architecture-level sparsity introduced by MoEs does not interfere with our intervention method. To verify our performance on this MoE architecture, we conduct experiments with a sparse MoE model, OLMoE-1B-7B-0924, from \citet{muennighoff2024olmoe}. We provide the results of our method, ITI, and base model for the TruthfulQA dataset as follows.

Table~\ref{tab:moe_results} shows that RADIANT works well for the OLMoE-1B-7B-0924 model, boosting the True score by nearly 8\%. RADIANT outperforms ITI regarding True * Info, True, MC1, and MC2. This showcases its effectiveness in generating truthful answers in architectures like Sparse MoEs.

\begin{table*}[!ht]
\centering
\begin{small}
\begin{tabular}{lcccccc}
\toprule
\textbf{Methods} & \textbf{True * Info (\%) $\uparrow$} & \textbf{True (\%) $\uparrow$} & \textbf{MC1 $\uparrow$} & \textbf{MC2 $\uparrow$} & \textbf{CE $\downarrow$} & \textbf{KL $\downarrow$} \\ \midrule
Unintervened & 23.43 & 31.42 & 28.92 & 45.42 & 2.67 & 0.00 \\
ITI & 25.52 & 32.74 & 28.48 & 45.50 & 2.84 & 0.13 \\
RADIANT (ours) & \textbf{28.56} & \textbf{39.17} & \textbf{30.13} & \textbf{47.05} & 2.68 & 0.12 \\ \bottomrule
\end{tabular}
\end{small}
\caption{Performance on TruthfulQA with OLMoE-1B-7B-0924.}
\label{tab:moe_results}
\end{table*}

\subsection{Comparison with Supervised Fine-Tuning}\label{ssec:radiant-vs-sft}

Supervised fine-tuning (SFT) attempts to align LLMs with human preferences~\citep{ouyang2022training}. Given a prompt, SFT encourages the model to generate desirable answers and reduce the likelihood of generating undesirable answers by optimizing the cross-entropy loss. However, SFT's requirement to fine-tune all LLM parameters demands substantial GPU resources for the back-propagation operations. Due to computational constraints, we can only perform SFT on the GPT2-large, the smallest model in our experiments. 

\begin{table}[h]
\centering
\begin{tabular}{{|c|c|}}
\hline
\textbf{Parameter}    & \textbf{Value}          \\ \hline
\texttt{learning\_rate}    & 0.00002                \\ \hline
\texttt{weight\_decay}    & 0                     \\ \hline
\texttt{adam\_beta1}      & 0.8                   \\ \hline
\texttt{adam\_beta2}      & 0.999                 \\ \hline
\texttt{adam\_epsilon}    & $1 \times 10^{-8}$    \\ \hline
\texttt{max\_grad\_norm}  & 1                     \\ \hline
\texttt{batch\_size}  & 32                     \\ \hline
\texttt{epochs\_num}  & 5                     \\ \hline
\texttt{lr\_scheduler\_type}  & linear \\ \hline
\end{tabular}
\caption{Parameter values for SFT.}
\label{tab:parameters}
\end{table}

Table~\ref{table:benchmark-others} highlights the advantages of inference-time methods like ours: by avoiding gradient computation or backpropagation, they offer a lightweight, fast, versatile, and economical way to improve the performance of LLMs. This is especially useful in low-resource scenarios. Because Llama-7B is used as a base model for many of our experiments, we also include the results of SFT on Llama-7B for comparison, but it is worth noting that the number is taken from the ITI paper~\cite{ref:li2024inference}. Since our evaluation framework differs from ITI in terms of the GPT-judge and GPT-info models, which is attributed to the fact that these models in the ITI paper are no longer available on OpenAI, the results may not be fair for comparison. From Table~\ref{sec:other_models}, SFT achieves the best performance in terms of MC metrics and reaches a high score of True * Info and True. Regarding the True score, RADIANT still outperforms SFT in the individual and integrating versions with FSP, offering 38.73\% and 40.41\% correct answers, respectively. When combined with FSP, RADIANT achieves 35.36\% in True * Info score, surpassing SFT but requiring fewer resources. For the implementation of SFT, we use the SFTTrainer framework\footnote{\url{https://huggingface.co/docs/trl/en/sft_trainer}}, one of the most popular frameworks for this algorithm. While we remained almost the default parameters proposed by the library, we had to tune many important parameters like learning rate, parameters of the Adam optimizer, weight decay, and so on, to get a consistent and stable fine-tuned model. Some important parameters for SFT are reported in Table~\ref{tab:parameters}, while its best performance is shown previously in Table~\ref{table:benchmark-others}. This observation strongly supports the practicability of RADIANT, which only necessitates tuning two key hyperparameters $\alpha$ in the probe loss~\eqref{eq:classifier}, and $\Gamma = \Phi^{-1}(1-\gamma)$ in the computation of the intervention map~\eqref{eq:sdp}. A detailed analysis of these parameters to provide insight into their impact is presented in Appendix~\ref{sec:analysis_hype}. This section offers useful insights and guidelines for selecting values for any new models. Furthermore, compared to other methods like ITI, the grid search on two hyperparameters like ours is efficient and reasonable, so it is not harder to tune the hyperparameters of RADIANT than other previous works.

\section{Experiments on Other NLP Tasks}

\subsection{Toxicity Mitigation Benchmark} \label{sec:tox_exp_setting}

In this section, we show the performance of RADIANT in mitigating toxicity in long-form text generation. In this task, the language models are required to complete an incomplete prefix piece of a text. Normally, the prefix prompt is selected to elicit toxic content from LLMs. For a fair comparison to previous works, we set up experiments following \citet{singhrepresentation} and \citet{pozzobon2023goodtriever}, which is detailed below.

\noindent \textbf{Training dataset.} We use the Toxic Comments Classification Challenge data.\footnote{\url{https://www.kaggle.com/c/jigsaw-toxic-comment-classification-challenge}} The dataset comprises sentences and their human toxicity labels. We follow the data preprocessing steps from~\citet{singhrepresentation} while the activations gathering is identical to the procedure of the QA task.

\noindent \textbf{Models.} Following existing works in the field, we adopt the GPT2-Large as the base model across all experiments of the toxicity mitigation task.

\noindent \textbf{Hyperparameter} As mentioned in the QA task section, there are two important hyperparameters in our framework, namely $\alpha$, and $\Gamma = \Phi^{-1}(1-\gamma)$, which would be selected by a grid search procedure detailed in Appendix~\ref{sec:analysis_hype}.

\noindent \textbf{Baselines.} We include several baselines that have the same goal of reducing the toxicity of LLMs, including MIMIC~\cite{singhrepresentation}, DEXPERTS~\cite{liu2021dexperts}, DAPT~\cite{gururangan2020don}, UDDIA~\cite{yang2022unified}, PPLM~\cite{dathathri2019plug}, GOODTRIEVER~\cite{pozzobon2023goodtriever}. As for MIMIC, we consider two versions: Mean Matching (MM) and Mean+Covariance Matching (MCM), introduced in their original paper.

\noindent \textbf{Metrics.} We assess the performance using three key metrics: toxicity, fluency, and diversity.
\begin{table}[!h]
\centering
\begin{small}
\begin{tabular}{ll}
\toprule
\textbf{Hyperparameter}       & \textbf{Value} \\ \midrule
Number of Samples             & 25                  \\
Max Length                    & 20                  \\
Temperature                   & 1                   \\
Top-p (sampling)              & 0.9                 \\
Top-k (sampling)              & 0             \\ \bottomrule
\end{tabular}
\end{small}
\caption{Hyperparameter settings for the decoding mechanism in the toxicity mitigation task}
\label{table:hyperparameter-settings}
\end{table}
\begin{enumerate}[label=(\roman*), leftmargin=5mm]
    \item Toxicity: we use the non-toxic split of RealToxicityPrompts \cite{gehman2020realtoxicityprompts} and utilize the evaluation framework in \citet{liu2021dexperts} and \citet{singhrepresentation}. For each prompt in the dataset, the models generate 25 outputs, each capped at 20 tokens in length. The parameters of the shared decoding mechanism of all algorithms are presented in Table~\ref{table:hyperparameter-settings}. These outputs are analyzed using Perspective API,\footnote{\url{https://perspectiveapi.com/}} which estimates the likelihood that a human would perceive the text as toxic. Two metrics are derived:
\begin{itemize}[leftmargin=5mm]
\item {Expected Maximum Toxicity is denoted as Exp.~Max.~Tox. We identify the output with the highest toxicity score for every prompt and compute the average of these maximum scores across all prompts.}
\item {Toxic Completion Proportion is abbreviated as Tox.~Prob. This metric tracks the fraction of outputs considered toxic, where toxicity is defined as a score above 0.5 based on the Perspective API's threshold.}
\end{itemize}

\item Fluency is evaluated by calculating the perplexity of the generated outputs, using GPT-2 (XL) as a reference model. Lower perplexity values suggest that the text is more coherent and grammatically fluent. 

\item Diversity is assessed by examining the ratio of unique n-grams (1-gram, 2-gram, and 3-gram) to the total number of tokens in the generated text. This metric captures the range of variation in the outputs, with higher values indicating more diverse and varied language use. This methodology ensures a balanced evaluation, providing insights into the ability of models for non-toxic, fluent, and diverse generation.
\end{enumerate}

\noindent \textbf{Results.} The experimental results of the baselines are shown in Table~\ref{table:benchmark-toxic}, where the base model for all methods is GPT-2 Large. The result of the original model is described in the first row. We divide the baselines into two groups. Using an extensive fine-tuning procedure, the first group comprises DAPT, GeDI, PPLM, UDDIA, DExperts, and GOODTRIEVER. In contrast, the second group contains inference-time fine-tuning-free methods like MIMIC, ITI, and RADIANT. The baselines in the first group are better than their counterparts in the second group regarding toxicity metrics. However, these methods require fine-tuning or computing gradients at inference time, which can be computationally intensive. MIMIC, ITI, and RADIANT achieved a toxicity reduction comparable to many algorithms in the first group, but consumed much fewer resources. Specifically, RADIANT is superior to PPLM and is equally competitive to DAPT. In particular, RADIANT offers the best toxicity reduction impact within the second group compared to ITI and MIMIC while maintaining a better fluency and diversity of generated sentences. The fluency of RADIANT is even more favored than almost all algorithms in the first group, except for UDDIA. At the same time, its diversity metric is better than that of other baselines except for PPLM.

\begin{table*}[t]
\centering
\begin{small}
\begin{tabular}{lcccccc}
\toprule
Model   & Exp. Max. Tox. $\downarrow$ & Tox. Prob. $\downarrow$ & Fluency $\downarrow$ & 1-gram $\uparrow$ & 2-gram $\uparrow$ & 3-gram $\uparrow$ \\ \midrule
GPT-2 (large)                    & 0.39          & 0.25          & 24.66 & 0.58          & 0.85          & 0.85          \\ \midrule
DAPT                             & 0.27          & 0.09          & 30.27 & 0.57          & 0.84          & 0.84          \\
GeDI                             & 0.24          & 0.06          & 48.12 & \textbf{0.62}          & 0.84          & 0.83          \\
PPLM (10\%)                      & 0.38          & 0.24          & 32.58 & 0.58          & \textbf{0.86}          & \textbf{0.86}          \\
UDDIA                            & 0.24          & 0.04          & \textbf{26.83} & 0.51          & 0.80          & 0.83          \\
DExperts         & \textbf{0.21} & \textbf{0.02} & 27.15 & 0.56          & 0.84          & 0.84          \\
GOODTRIEVER                      & 0.22          & 0.04          & 27.11 & 0.58          & 0.82          & 0.83          \\ \midrule
MM (MIMIC)            & 0.33          & 0.16          & 28.00 & \textbf{0.58} & \textbf{0.85} & \textbf{0.85} \\
MCM (MIMIC) & 0.29          & \textbf{0.09} & 30.70 & 0.54          & 0.84          & 0.84          \\
ITI                              & 0.31          & 0.12          & 33.12 & 0.57          & \textbf{0.85} & \textbf{0.85} \\
RADIANT & \textbf{0.27}               & \textbf{0.09}           & \textbf{27.10}       & \textbf{0.58}     & \textbf{0.85}     & \textbf{0.85}\\ \bottomrule    
\end{tabular}
\caption{Quantitative results of different intervention methods on RealToxicityPrompts dataset. Parameters of RADIANT: $\alpha = 2.5, \Gamma = 15$.}
\label{table:benchmark-toxic}
\end{small}
\end{table*}

\subsection{Long-Fact Generation Benchmark} \label{sec:long-fact}

We inspected the performance of our method and ITI on the Long-Fact task. We followed the experimental setup and prompt format in \citet{cheng2024integrative} and used their code repository\footnote{\url{https://github.com/YiCheng98/IntegrativeDecoding}} for evaluation. This task requires models to generate detailed document-length descriptions of queried objects, often exceeding 1,000 tokens. Evaluation involves breaking responses into atomic facts using LLaMA3.1-70B-Instruct and evaluating their truthfulness with GPT-4. Metrics include Precision (truthful fact proportion), Recall@128 (truthful facts per 128), and F1@128. Results are based on 120 samples. Because we do not have a set of true and wrong samples in this dataset, we cannot learn the mapping for ITI and RADIANT. Therefore, we transfer the learned mappings of ITI and RADIANT from Truthful QA for Qwen2-7B-Instruct and Mistral-7B-Instruct-v0.2. 

The results for Qwen2-7B-Instruct are reported in Table~\ref{tab:long_fact_qwen}, while the results for Mistral-7B-Instruct-v0.2 are reported in Table~\ref{tab:long_fact_mistral}.

The results demonstrate that RADIANT's intervention mapping effectively enhances truthfulness in the long-fact task, showcasing strong generalization across models. In contrast, ITI's mapping transfers well to Qwen2-7B-Instruct but performs poorly on Mistral-7B-Instruct-v0.2, highlighting its limited adaptability. RADIANT achieves the highest precision across both models, indicating its strength in generating truthful facts. While ITI slightly outperforms in recall for Qwen2-7B, and the unintervened model leads in recall for Mistral-7B, RADIANT maintains a strong balance between precision and recall. This is reflected in its highest F1@128 scores, making it the most effective overall. The small variations in recall suggest that interventions refine the accuracy of the facts rather than significantly increasing the number of truthful facts generated.

\subsection{Creative Writing Benchmark} \label{sec:creative_writing}

To further address the concern about meaning shift, we tested our method on the Creative Writing v3 task from the EQ-Bench benchmark~\cite{paech2023eq} using Qwen2-7B-Instruct. This experiment evaluates how intervention mapping affects the meaning of creative text. The ELO scores in Table~\ref{tab:creative_writing} show a slight decrease when applying ITI and RADIANT, which is expected, since these interventions enhance truthfulness, potentially limiting word choice flexibility. However, the results indicate that the interventions do not significantly harm the model's creative capabilities.

\begin{table*}
\centering
\begin{small}
\begin{tabular}{lc}
\toprule
\textbf{Methods} & \textbf{ELO $\uparrow$} \\ \midrule
Unintervened & 592.5 \\
ITI & 554.6 \\
RADIANT (ours) & \textbf{578.6} \\ \bottomrule
\end{tabular}
\end{small}
\caption{ELO Scores on Creative Writing v3 Task.}
\label{tab:creative_writing}
\end{table*}

\begin{table*}[!htbp]
\centering
\begin{small}
\begin{tabular}{lccc}
\toprule
\textbf{Methods} & \textbf{Precision $\uparrow$} & \textbf{Recall@128 $\uparrow$} & \textbf{F1@128 $\uparrow$} \\ \midrule
Unintervened & 87.54 & 55.42 & 69.02 \\
ITI & 88.01 & \textbf{56.73} & 68.50 \\
RADIANT (ours) & \textbf{88.52} & 56.17 & \textbf{70.13} \\ \bottomrule
\end{tabular}
\end{small}
\caption{Performance on the Long-Fact task with Qwen2-7B-Instruct.}
\label{tab:long_fact_qwen}
\end{table*}

\begin{table*}[!htbp]
\centering
\begin{small}
\begin{tabular}{lccc}
\toprule
\textbf{Methods} & \textbf{Precision $\uparrow$} & \textbf{Recall@128 $\uparrow$} & \textbf{F1@128 $\uparrow$} \\ \midrule
Unintervened & 88.18 & \textbf{59.23} & 72.15 \\
ITI & 87.82 & 58.73 & 72.20 \\
RADIANT (ours) & \textbf{89.03} & 59.22 & \textbf{73.73} \\ \bottomrule
\end{tabular}
\end{small}
\caption{Performance on the Long-Fact task with Mistral-7B-Instruct-v0.2.}
\label{tab:long_fact_mistral}
\end{table*}

\section{Mathematical Proof} \label{sec:proof}

\begin{proof}[Proof of Theorem~\ref{thm:intervene}] The logistic classifier $\mathcal C_{\ell h}$ output a prediction 0 if $\vartheta_{\ell h} + \theta_{\ell h}^\top a_{\ell h} < 0$. If $\QQ_{\ell h}$ is Gaussian $\mc N(\m, \cov)$, then by~\cite[Theorem~10.4.1]{ref:prekopa1995stochastic}, the probability constraint of~\eqref{eq:intervention} can be written as
\[
 \vartheta_{\ell h} + \theta_{\ell h}^\top \mu + \Phi^{-1}(1-\gamma) \sqrt{\theta_{\ell h}^\top \cov \theta_{\ell h}} \le 0.
\]
Next, we add an auxiliary variable $t \in \R_+$ with an epigraph constraint
$\sqrt{\theta_{\ell h}^\top \cov \theta_{\ell h}} \le t$. Because $\Phi^{-1}(1-\gamma) > 0$ for $\gamma \in (0, 0.5)$, problem~\eqref{eq:intervention} is equivalent to
\begin{equation} \notag
    \begin{array}{cl}
         \min & \| \m - \msa \|_2^2 + \| \cov^\half - \covsa^{\half} \|_F^2 \\
         \st & \vartheta_{\ell h} + \theta_{\ell h}^\top \mu + \Phi^{-1}(1-\gamma) t \le 0, \\ &\sqrt{\theta_{\ell h}^\top \cov \theta_{\ell h}} \le t \\
         & \mu \in \R^d,~\cov \in \PSD^d,~t \in \R_+.
    \end{array}
\end{equation}
Let $S \leftarrow \cov^{\half} \in \PSD^d$, the constraint $\sqrt{\theta_{\ell h}^\top \cov \theta_{\ell h}} \le t$ is equivalent to $\| S \theta_{\ell h} \|_2 \le t$, which leads to~\eqref{eq:sdp}. Thus, the optimal pushforward $\Delta_{\ell h}$ should push $\Pnom_{\ell h} \sim \mc N(\msa, \covsa)$ to $\QQ_{\ell h} \sim \mc N(\m\opt, (S\opt)^2)$. One can verify through simple linear algebraic calculations that the mapping $\Delta_{\ell h}(a_{\ell h}) = G\opt_{\ell h} a_{\ell h} + g\opt_{\ell h}$ defined in the theorem statement is the desired mapping. This completes the proof. 
\end{proof}

\end{document}